\documentclass[letterpaper, 10 pt, conference]{ieeeconf}  

\IEEEoverridecommandlockouts                              
\usepackage[utf8]{inputenc}
\usepackage[T1]{fontenc}

\usepackage{cite}
\usepackage[pdftex]{graphicx}
\usepackage{amsmath}
\usepackage{amssymb}
\usepackage{array}
\usepackage{wasysym}
\usepackage{bm}
\usepackage{siunitx}
\usepackage{url}
\usepackage{color}
\usepackage{epstopdf}
\usepackage[ruled]{algorithm2e} 
\usepackage{arydshln}
\usepackage{mathtools}

\usepackage{stfloats}
\usepackage{relsize}
\usepackage{cancel}
\usepackage{epstopdf}
\usepackage{environ}
\usepackage{amsthm}

\NewEnviron{ruledtable}{%
\par\addvspace{2mm}\hrule height 0.03cm 
\begin{table}[!h]\BODY\end{table}
\hrule height 0.03cm \addvspace{2mm}
}

\DeclareMathOperator*{\argmin}{arg\,min}

\theoremstyle{plain}

\newtheorem{definition}{Definition}
\newtheorem{proposition}{Proposition}
\newtheorem{remark}{Remark}
\newtheorem{corollary}{Corollary}
\newtheorem{assumption}{Assumption}

\usepackage{CJKutf8}

\renewcommand{\baselinestretch}{0.98}

\begin{document}
\title{\LARGE \bf{Hierarchical Relaxation of Safety-critical Controllers: Mitigating \\  Contradictory Safety Conditions with Application to Quadruped Robots}}
\author{Jaemin Lee, Jeeseop Kim, and Aaron D. Ames
\thanks{This work is supported by Dow under the project  \#227027AT.}
\thanks{The authors are with the Department of Mechanical and Civil Engineering, California Institute of Technology, Pasadena, CA, USA (e-mail: \{jaemin87, jeeseop, ames\}@caltech.edu)} 
}

\maketitle

\begin{abstract}
The safety-critical control of robotic systems often must account for multiple, potentially conflicting, safety constraints.  
This paper proposes novel relaxation techniques to address safety-critical control problems in the presence of conflicting safety conditions. 
In particular, Control Barrier Function (CBFs) provide a means to encode safety as constraints in a Quadratic Program (QP), wherein multiple safety conditions yield multiple constraints. 
However, the QP problem becomes infeasible when the safety conditions cannot be simultaneously satisfied. To resolve this potential infeasibility, we introduce a hierarchy between the safety conditions and employ an additional variable to relax the less important safety conditions (Relaxed-CBF-QP), and formulate a cascaded structure to achieve smaller violations of lower-priority safety conditions (Hierarchical-CBF-QP). The proposed approach, therefore, ensures the existence of at least one solution to the QP problem with the CBFs while dynamically balancing enforcement of additional safety constraints. Importantly, this paper evaluates the impact of different weighting factors in the Hierarchical-CBF-QP and, due to the sensitivity of these weightings in the observed behavior, proposes a method to determine the weighting factors via a sampling-based technique.  The validity of the proposed approach is demonstrated through simulations and experiments on a quadrupedal robot navigating to a goal through regions with different levels of danger. 
\end{abstract}


\section{Introduction}
\label{sec:introduction}
\vspace{-0.1cm}
Robotic systems are being increasingly deployed to perform a wide-variety tasks in unstructured environments, including applications that necessitate operating in close proximity with people. 
Therefore, ensuring the safety of these systems while they operate in such environments is crucial.  As such, there have been a variety of approaches to safety applied to a range of robot types, including: mobile robots\cite{igarashi2018collision,desai2022clf,manjunath2021safe}, legged robots\cite{teng2021toward,liu2023realtime,lee2021efficient}, multi-agent systems\cite{wang2017safety,chen2020guaranteed}, and even humanoid robots \cite{kim2020dynamic,khazoom2022humanoid,koptev2021real}. As of late, Control Barrier Functions (CBFs) \cite{7782377} have emerged as a popular method for practically enforcing safety constraints on robotic systems. However, as the environments and tasks become more complicated, safety constraints become contradictory or conflicting. Therefore, it is essential to manage such contradictory safety statements to ensure safe and performant controllers in obstacle-rich environments as illistrated in Fig. \ref{Fig1}. With this motivation, this paper addresses the potential infeasibility of Quadratic Programming (QP) that include contradictory CBF constraints and proposes a hierarchical approach for synthesizing feedback controllers and safety filters handling multiple safety constraints.

\begin{figure}[t] 
\centering
\includegraphics[width=\linewidth]{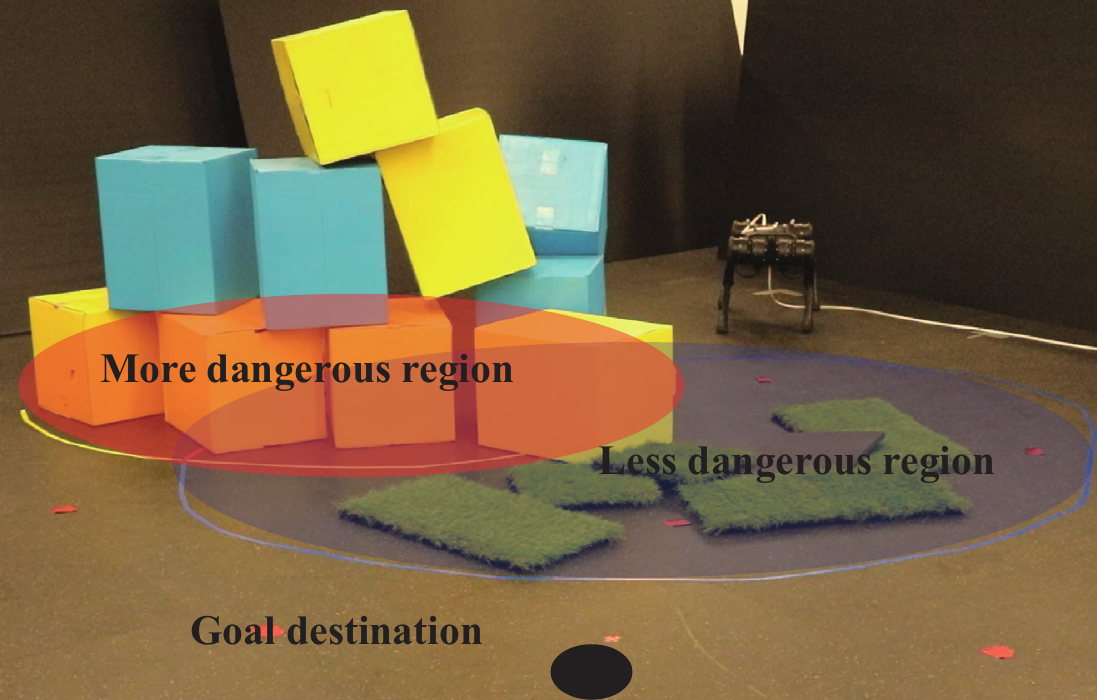}
\vspace{-4mm}
\caption{\textbf{Motivation of our study}: When robots operate in complex environments with numerous obstacles, they are often required to follow trajectories that pass through unsafe regions. Such scenarios give rise to challenges in control and planning to ensure the safety of the robots.}
\label{Fig1}
\vspace{-0.5cm}
\end{figure}

\subsection{Related Work}
For robotic applications, CBF-based approaches have proven to be an effective method for generating safe robot motions in the context of planning, control, and learning. To provide an example, for legged robots safety-critical CBF constraints can be enforced to obtain safe foot placements \cite{agrawal2017discrete}, coupled with learning \cite{csomay2021episodic}, and combed with MPC in a multi-rate fashion \cite{grandia2021multi}.  Manipulators can be controlled with CBFs to enforce box constraints in operational space \cite{rauscher2016constrained}, or avoid collisions with the environment \cite{ferraguti2022safety,singletary2022safety}. In the whole-body control of humanoid robots, joint limits and self-collision avoidance constraints can be imposed using CBFs \cite{khazoom2022humanoid}.  However, these studies assume the existence of a safe set satisfying the safety-critical constraints over the entire time horizon. In reality, many robots operate with contradictory safety statements. This fundamental issue dates back to the control theoretic origins of CBFs. 



At their inception \cite{ames2014control,7782377}, CBFs were introduced as a control theoretic tool to ensure safety framed as forward set invariance.  A CBF results in an affine inequality constraint in the input that, when enforced, implies safety; since this constraint in affine, it can be enforces in a QP with a cost that minimizes the difference between the desired input and the safe input---the QP nature means it can be solved in real-time.  CBFs can also naturally be combined with Control Lyapunov Functions (CLFs), since these again yield an inequality constraint affine in the input---CLFs and CBF can thus be utilized in a single QP to enforce stability and safety under the assumption of feasibility.  In the case when the CLF and CBF constraints conflict, the CLF constraint is typically relaxed to strictly fulfill the safety-critical constraints at the cost of goal attainment.  This interplay between conflicting CLFs and CBFs was studied in \cite{zeng2021safety} by relaxing CLF constraint and scaling the lower bound of CBF constraint.  This naturally leads to the setting of multiple barrier functions.


Multiple barrier functions have been widely studied in the context of multi-agent robotic systems, with Boolean compositions of safety-critical constraints formulated using sums and products of barrier functions \cite{wang2016multi}. For multi-agent systems, non-smooth barrier functions can also be used with the max and min operators for Boolean compositions \cite{glotfelter2017nonsmooth}. Linear Temporal Logic (LTL) tasks are achieved by satisfying multiple time-varying barrier functions \cite{lindemann2018control}, and the approach has been extended to multi-agent systems considering conflicts of Temporal Logic tasks \cite{lindemann2019control}. However, none of these approaches consider a hierarchy among the safety-critical constraints. Recently, a decoupling method for multiple CBF constraints has been proposed to compose multiple CBFs while enforcing input constraints \cite{breeden2022compositions}. However, this method does not impose a hierarchy of safety-critical constraints either.
This paper, therefore, considers potentially conflicting CBFs, and propose a relaxation achieved by imposing a hierarchy based on their relative importance. 

\subsection{Contributions}
This paper leverages a two-layered control architecture for robotic systems that includes a reduced-order model (ROM) to generate control inputs that result in safe paths, and a full-order model (FOM) that leverage the nonlinear dynamics of the system to track these paths. The proposed approach focuses on principled means of relaxing hierarchical safety-critical constraints, expressed as CBFs in a QP.  In particular, we begin with the assumption that multiple safety conditions cannot be simultaneously satisfied and that a hierarchy exists among them.
The goal is then to strictly fulfill the top-priority safety condition while minimizing lower-priority violations, via a hierarchical architecture similar to \cite{lee2022hierarchical}. Additionally, the approach analyzes the weighting factors between the feedback control input error and the relaxation variable and proposes a technique to determine these weighting in the context of achieving a given task. 


The main contributions of this paper are threefold. First, the proposed approach ensures that the formulated Hierarchical-CBF-QP is always feasible while strictly satisfying the top-prioritized safety-critical constraint. In addition, we reduce the violations of the lower-prioritized CBFs in order of their importance.  Second, the approach includes a technique to determine the weighting factors between the nonlinear feedback control input error and the relaxation variables. Instead of defining new candidate functions that lower bound the CBF constraints (extended class $\mathcal{K}$ functions), we still utilize the same functions and obtain optimal parameters of the weightings in Hierarchical-CBF-QP, minimizing the cost function.  Lastly, the proposed method enables to the generation of safer paths without the need for re-planning trajectories, particularly in crowded obstacle-rich environments.  We demonstrate this experimentally. 

Our paper is organized as follows. In Section \ref{section2}, provide the necessary background on the nonlinear modelling and control of robots, and frame the safety-critical control approach, including the introduction of CBFS and safety filters. Section \ref{section3} outlines the proposed approach for relaxing safety-critical controllers and thereby incorporating contradictory CBF conditions. This section also includes a technique for determining the appropriate weighting factors to optimize performance. The proposed approach is also extended to handle arbitrary CBF constraints in Section \ref{section3}. Finally, Section \ref{section4} provides both simulation and experiment results of the proposed approach applied to a quadrupedal robot platform (A1) walking in a cluttered environment. 

\begin{figure}[t] 
\centering
\includegraphics[width=\linewidth]{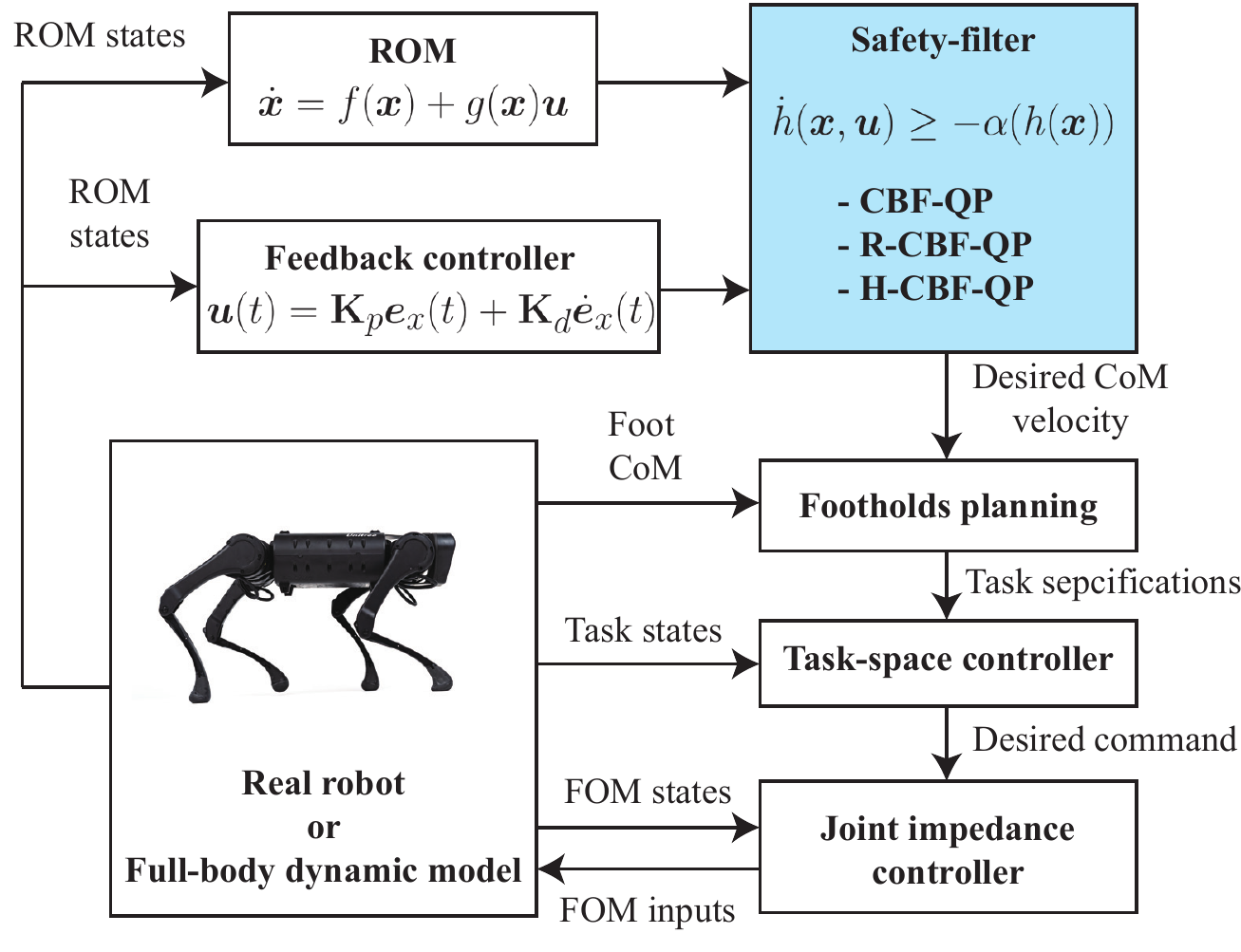}
\vspace{-6.5mm}
\caption{\textbf{Control Framework}: Our relaxation approaches are crucial components of the control framework described earlier. To ensure a fair comparison of these methods, we have used consistent components and setups in our implementation.}
\label{Fig2}
\vspace{-0.5cm}
\end{figure}

\section{Problem Statement}
\label{section2}
This section covers the fundamental concepts that underpin the proposed control architecture for robotic systems described in Fig. \ref{Fig2}: the control-affine system (ROM), the rigid-body dynamic model (FOM), the feedback controllers, and the concept of CBF. With these concepts in place, we define our problem, making proper assumptions to ensure the effectiveness of the proposed control architecture.

\subsection{Reduced-order Model and Feedback control}
We introduce control affine systems defined in a state space $\mathcal{X} \subseteq \mathbb{R}^{n}$ and an admissible input space $\mathcal{U} \subseteq \mathbb{R}^{m}$: 
\begin{equation}
\label{eqn:nonlinearcontrolsys}
    \dot{\bm{x}}  = f(\bm{x}) + g(\bm{x})\bm{u}
\end{equation}
where $\bm{x}\in \mathcal{X}$ and $\bm{u} \in \mathcal{U}$ denote the state and control input. In addition, $f: \mathcal{X} \mapsto \mathbb{R}^{n}$ and $g: \mathcal{X} \mapsto \mathbb{R}^{m}$ are Lipschitz continuous. In this paper, the above control affine system is considered as ROM such as the linear inverted pendulum or double integrator models. To design a continuous feedback controller, we use  a simple PD control law with proportional and derivative gains, $\mathbf{K}_p$ and $\mathbf{K}_d$, given a state trajectory, $\bm{x}^{d}(t) ,\; \forall t \in [t_0, \;t_f]$ :   
\begin{equation}
    \bm{u}(t) = \mathbf{K}_p \bm{e}_x(t) + \mathbf{K}_{d} \dot{\bm{e}}_x(t)
\end{equation}
where $\bm{e}(t) = \bm{x}^{d}(t) - \bm{x}(t)$. This control input will be the desired input for planning and control based on the full-order model. The state of ROM is computed and updated based on measurements or the state update of FOM.  

\subsection{Full-order Model, Planning, and Task-Space Control}
To realize the control input, we employ full-body rigid body dynamics and intermediate planners that connect ROM and FOM. For instance, it is needed to plan the foot placement while implementing the control input $\bm{u}(t)$ with Raibert's heuristic \cite{raibert1984experiments}. Since the planning of footholds is outside the scope of this paper, we assume that the task trajectory $\bm{y}^{d}(t) \in \mathbb{R}^{n_y}$ is properly planned in terms of $\bm{u}(t)$.

The equation of motion for floating-base robots with contacts is typically presented as follows:
\begin{align}
	\mathbf{D}(\mathbf{q})\ddot{\mathbf{q}} + \mathbf{H}(\mathbf{q},\dot{\mathbf{q}}) &= \mathbf{S}^{\top} \bm{\tau} + \mathbf{J}_{c}(\mathbf{q})^{\top} \mathbf{F}_{c},  \label{eq:dyn_eq}\\
    \mathbf{J}_c(\mathbf{q})\ddot{\mathbf{q}}  + \dot{\mathbf{J}}_c(\mathbf{q},\dot{\mathbf{q}}) \dot{\mathbf{q}} &= \mathbf{0} \label{eq:holonomic_const}
\end{align}
where $\mathbf{q} \in \mathcal{Q} \subset \mathbb{R}^{n_q}$, $\mathbf{D}(\mathbf{q}) \in \mathbb{S}_{>0}^{n_q}$, $\mathbf{H}(\mathbf{q}, \dot{\mathbf{q}})\in \mathbb{R}^{n_q}$, $\mathbf{S} \in \mathbb{R}^{(n_q-6)\times n}$, and $\bm{\tau} \in \Gamma \subset \mathbb{R}^{n_q-6}$ denote the joint variable, mass/inertia matrix, the sum of Coriolis/centrifugal and gravitational force, a selection matrix, and a control torque command, respectively. In addition, $\mathbf{F}_{c} \in \mathbb{R}^{n_c}$ and $\mathbf{J}_{c}(\mathbf{q})\in \mathbb{R}^{n_c \times n_q}$ represent the contact wrench and the corresponding Jacobian, respectively. For rigid contacts, we consider an equality constraint \eqref{eq:holonomic_const} in the acceleration level. With $\bm{y}^{d}$, $\dot{\bm{y}}^{d}$, $\mathbf{q}$, and $\dot{\mathbf{q}}$, we formulate a QP problem to obtain the desired control command to minimize the tracking error: 
\begin{ruledtable}
\vspace{-2mm}
{\textbf{\normalsize Task-Space Controller:}}
\par\vspace{-4mm}{\small 
\begin{subequations}
\begin{align*}
    \min_{(\ddot{\mathbf{q}}, \bm{\tau}, \mathbf{F}_{c})} &\quad \mathcal{J}_{y} = \| \mathbf{K}_p(\bm{y}^{d} - \bm{y}) +  \mathbf{K}_d (\dot{\bm{y}}^{d} - \dot{\bm{y}}) - \mathbf{J}_y(\mathbf{q}) \ddot{\mathbf{q}} \|^2, \\
    \textrm{s.t.}&\quad \mathbf{D}(\mathbf{q})\ddot{\mathbf{q}} + \mathbf{H}(\mathbf{q}, \dot{\mathbf{q}}) = \mathbf{S}^{\top} \bm{\tau} + \mathbf{J}_{c}(\mathbf{q})^{\top} \mathbf{F}_{c}, \\
    &\quad \mathbf{J}_c(\mathbf{q})\ddot{\mathbf{q}}  + \dot{\mathbf{J}}_c(\mathbf{q},\dot{\mathbf{q}}) \dot{\mathbf{q}} = 0, \\
    &\quad \mathbf{F}_{c} \in \mathcal{CWC} (\mathbf{q}), \quad \bm{\tau} \in \Gamma
\end{align*}
\end{subequations}}
\vspace{-6mm}
\end{ruledtable}
\noindent
$\mathcal{CWC}(\mathbf{q})$ is the set of the contact forces satisfying the contact wrench cone constraint. The control of the legged robots is achieved through a joint-impedance controller, which integrates the joint feedback control with a feedforward torque obtained from the solution to the aforementioned QP. 

\subsection{Control Barrier Functions (CBFs)}
This section provides a brief overview of the fundamental concepts and definitions related to CBFs. We introduce a safe set $\mathcal{C} \subset \mathbb{R}^{n}$ defined as the $0$-superlevel set of a continuously differentiable function $h: \mathbb{R}^{n} \mapsto \mathbb{R}$, its boundary and interior sets:
\begin{equation}
\label{eqn:Cset}
    \begin{split}
        \mathcal{C} &\coloneqq \{\bm{x}\in \mathbb{R}^{n}: h(\bm{x}) \geq 0 \},\\
        \partial \mathcal{C} &\coloneqq \{ \bm{x} \in \mathbb{R}^{n}: h(\bm{x}) = 0 \},\\
        \textrm{int}(\mathcal{C}) &\coloneqq \{ \bm{x} \in \mathbb{R}^{n}: h(\bm{x}) > 0 \}.
    \end{split}
\end{equation}
The nonlinear system is safe with respect to the safe set $\mathcal{C}$ if $\mathcal{C}$ is forward invariant. To recall the basic concepts and definitions related to CBFs, a continuous function $\alpha: (a,b) \mapsto \mathbb{R}$ is an \emph{extended class $\mathcal{K}$ ($\mathcal{K}^e$) function} if it is strictly increasing with $\alpha(0) = 0$. In particular, if the open interval is $(-\infty, \infty)$ and $\alpha$ is strictly increasing with $\alpha(0) = 0$, $\lim_{r \rightarrow -\infty}\alpha(r) = - \infty$, and $\lim_{r\rightarrow +\infty} \alpha(r) = + \infty$, $\alpha$ is an extended class $\mathcal{K}_{\infty}$ ($\mathcal{K}_{\infty}^{e}$) function. Using the function $\alpha$, CBFs are defined according to \cite{7782377}.

\begin{definition}
Let $\mathcal{C} \subset \mathbb{R}^{n}$ be the $0$-superlevel set of a continuously differentiable function $h:\mathbb{R}^{n} \mapsto \mathbb{R}$ as in \eqref{eqn:Cset}.  
$h$ is a \textbf{Control Barrier Function (CBF)} if there exists a function $\alpha \in \mathcal{K}_{\infty}^{e}$ such that for all\footnote{Here we consider all points in $\mathbb{R}^{n}$ for notational simplicity.  Technically, one could consider a subset $E$ containing $\mathcal{C}$ and restrict our attention to this set $E$.  One would only need to check the CBF conditions on this set $E$.} $x\in \mathbb{R}^{n}$:
\begin{equation*}     \begin{split}
    \sup_{\bm{u} \in \mathbb{R}^{m}}  \dot{h}(\bm{x},\bm{u}) &=  \sup_{\bm{u}\in \mathbb{R}^{m}}  \left[ \mathcal{L}_{f}h(\bm{x}) + \mathcal{L}_{g}h(\bm{x}) \bm{u} \right] \geq - \alpha(h(\bm{x})).
    \end{split}
\end{equation*}
where $\mathcal{L}_{f} h(\bm{x}) = \frac{\partial h}{\partial \bm{x}}(\bm{x}) f(\bm{x})$ and $\mathcal{L}_{g} h(\bm{x}) = \frac{\partial h}{\partial \bm{x}}(\bm{x}) g(\bm{x})$ are Lie derivatives.
Given a CBF, we define a set of control input values for all $\bm{x} \in \mathbb{R}^{n}$ as follows:
\begin{equation*}
    \mathcal{K}_{\mathrm{cbf}}(\bm{x}) = \left\{\bm{u} \in \mathbb{R}^{m}:\mathcal{L}_{f}h(\bm{x}) + \mathcal{L}_{g}h(\bm{x}) \bm{u}  \geq - \alpha(h(\bm{x}))\right\}.
\end{equation*}
\end{definition}
Let $\bm{u}^{\star}$ be the desired control value.  We can find the (pointwise) closes safe control value through the use of a CBF framed in the context of the following QP: 
\begin{equation}\label{opt_second}
    \begin{split}
    \hat{\bm{u}} =& \argmin_{\bm{u} \in \mathcal{U}} \|\bm{u}^{\star} - \bm{u} \|^{2},\\
    & \textrm{s.t.} \quad \dot{h}(\bm{x}, \bm{u} ) \geq - \alpha(h(\bm{x})).
    \end{split}
\end{equation}
The key result of \cite{7782377} is that $\hat{\bm{u}}$ renders the set $\mathcal{C}$ forward invariant (that is, safe), i.e., for all $\bm{x}_0 \in \mathcal{C}$ it follows that $\bm{x}_0 \in \mathcal{C}$ for all time when $\hat{\bm{u}}$ is applied to \eqref{eqn:nonlinearcontrolsys}. 


\subsection{Problem Definition}
In this paper, let us consider two barrier functions $h_1: \mathbb{R}^{n} \mapsto \mathbb{R}$ and $h_2: \mathbb{R}^{n} \mapsto \mathbb{R}$, in turn the corresponding safe sets are defined as $\mathcal{C}_1$ and $\mathcal{C}_2$. Based on the defined CBFs, we formulate a standard CBF-QP given the feedback control input $\bm{u}^{\star}$.
\begin{ruledtable}
\vspace{-2mm}
{\textbf{\normalsize CBF-QP (Two CBFs):}}
\par\vspace{-4mm}{\small 
\begin{subequations}
\begin{align}
    \bm{u}^{\bullet} = \argmin_{\bm{u}\in \mathcal{U}} &\quad  \| \bm{u}^{\star} - \bm{u} \|^2, \nonumber \\
    \textrm{s.t.}&\quad \dot{h}_{1}(\bm{x}, \bm{u}) \geq - \alpha_{1}(h_1(\bm{x})), \label{cbf_1}\\
    &\quad \dot{h}_{2}(\bm{x}, \bm{u}) \geq - \alpha_{2}(h_2(\bm{x}))\label{cbf_2}
\end{align}
\end{subequations}}
\vspace{-6mm}
\end{ruledtable}
\noindent
When the feedback control command $\bm{u}^{\star}$ satisfies all constraints without conflicting each other, $\bm{u}^{\bullet}$ becomes original feedback control input $\bm{u}^{\star}$. However, the constraints frequently conflict with each other so we address specific assumptions for our problem, which is potentially infeasible CBF-QP. 

\begin{figure}[t] 
\centering
\includegraphics[width=\linewidth]{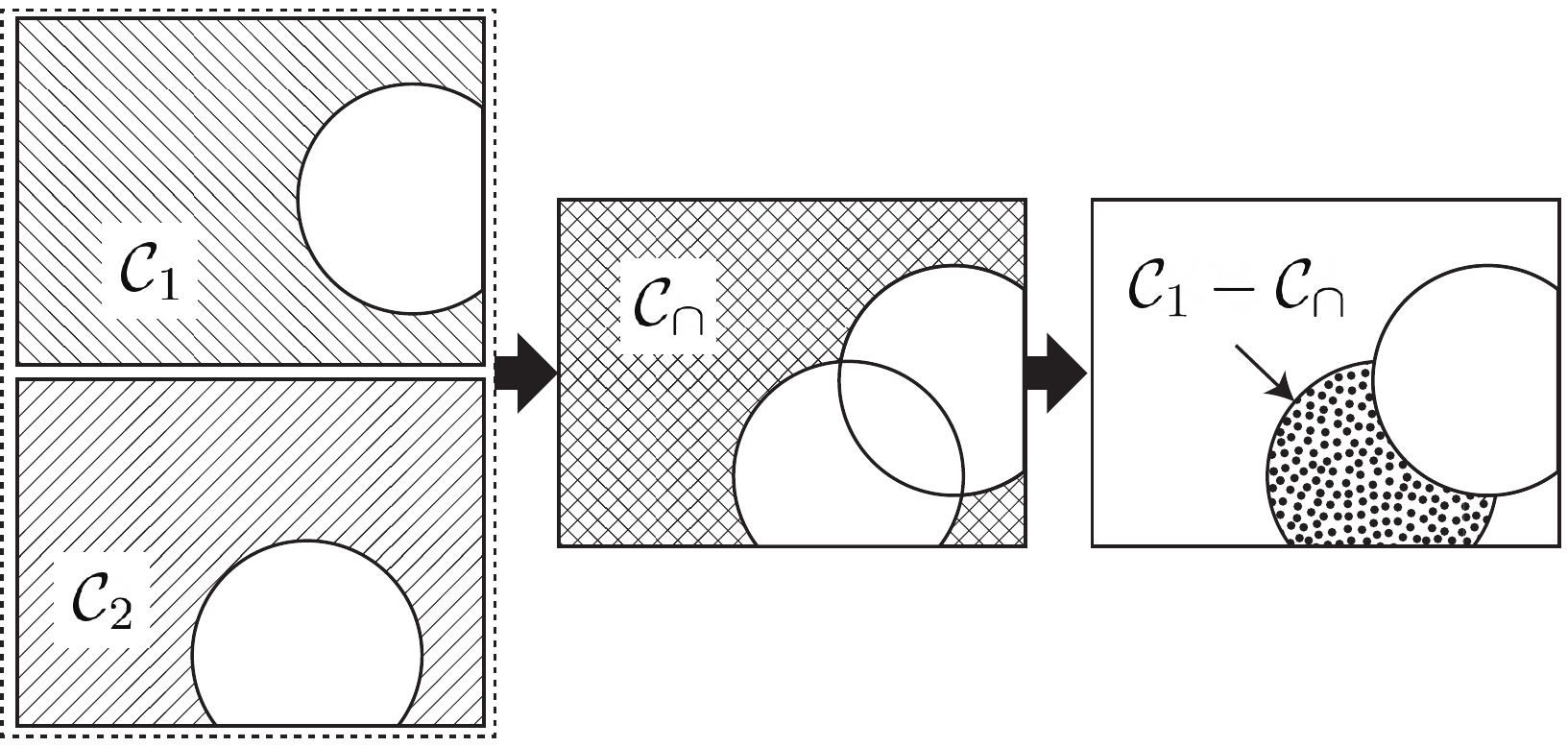}
\vspace{-6mm}
\caption{\textbf{Sets in our problem definition}: Given safety-critical constraints, the sets defined in this paper, $\mathcal{C}_1$, $\mathcal{C}_2$, $\mathcal{C}_{\cap}$, and $\mathcal{C}_1- \mathcal{C}_{\cap}$, are described visually.}
\label{Fig3}
\vspace{-0.5cm}
\end{figure}

To define our problem, we first define the following sets using $\mathcal{C}_1$ and $\mathcal{C}_2$:
\begin{equation}
    \begin{split}
        \mathcal{C}_{\cap} &= \left\{ \bm{x} \in \mathbb{R}^{n} : h_1(\bm{x}) \geq 0 \wedge h_2(\bm{x}) \geq 0 \right\},\\
        \mathcal{C}_{\cup} &= \left\{ \bm{x} \in \mathbb{R}^{n} : h_1(\bm{x}) \geq 0 \vee h_2(\bm{x}) \geq 0 \right\}
    \end{split}
\end{equation}
where $\wedge$ and $\vee$ denote logical operators AND and OR, respectively. It is true that $\mathcal{C}_{\cap} = \mathcal{C}_{1} \bigcap \mathcal{C}_2$ and $\mathcal{C}_{\cup} = \mathcal{C}_{1} \bigcup \mathcal{C}_2$. For proof of this, see Lemma 3.1 in \cite{wang2016multi}. One of the critical assumptions in our problem definition is as follows: 
\begin{assumption}
Consider two properly defined class $\mathcal{K}_{\infty}^{e}$ functions $\alpha_1$ and $\alpha_2$. We assume that there exists a state $\bm{x}$ along the trajectory such that $\bm{x} \in \mathcal{C}_{1} - \mathcal{C}_{\cap}$, which means $h_1(\bm{x})>0$ and $h_2(\bm{x})<0$, as shown in Fig. \ref{Fig3}.
\end{assumption}
\noindent
If the current state $\bm{x}$ satisfies Assumption 1, the potential for infeasibility arises due to a conflict between the constraints \eqref{cbf_1} and \eqref{cbf_2}, simultaneously.
\begin{proposition}
Under Assumption 1, if a state $\bm{x}$ is given, there exist $\alpha_1$ and $\alpha_2$ such that no control input from $\mathcal{K}_{\mathrm{cbf},1}(\bm{x})$ belongs to $\mathcal{K}_{\mathrm{cbf},2}(\bm{x})$.
\end{proposition}
\begin{proof}
Let us consider $h_2(\bm{x}) = -\beta_{1}(\bm{x}) h_1(\bm{x}) + \beta_2(\bm{x})$ where $\beta_1(\bm{x})>0$ and $\beta_2(\bm{x}) < \beta_1(\bm{x}) h_1(\bm{x})$. In addition, it is noted $\beta_1$ and $\beta_2$ are the class $C^{1}$ functions. The time derivatives of $\beta_1(\bm{x})$ and $\beta_2(\bm{x})$ becomes as follows:
\begin{equation*}
   \frac{ \partial {\beta}_{1,2}}{\partial \bm{x}} \dot{\bm{x}} =   \frac{ \partial {\beta}_{1,2}}{\partial \bm{x}} \left( f(\bm{x}) + g(\bm{x}) \bm{u} \right) = \dot{\beta}_{1,2}(\bm{x}, \bm{u}).
\end{equation*}
Then, the time derivative of $h_2(\bm{x})$ becomes
\begin{equation*}
\begin{split}
 \dot{h}_{2}(\bm{x},\bm{u}) &= -\dot{\beta}_1(\bm{x}, \bm{u}) h_1(\bm{x}) - \beta_1(\bm{x}) \dot{h}_{1}(\bm{x}, \bm{u}) + \dot{\beta}_{2}(\bm{x}, \bm{u}) \\
 &\leq -\dot{\beta}_{1}(\bm{x}, \bm{u}) h_1(\bm{x}) + \beta_1(\bm{x}) \alpha_{1}(h_1(\bm{x})) + \dot{\beta}_{2}(\bm{x}, \bm{u}).
\end{split}
\end{equation*}
If the defined $\alpha_2$ satisfies the following inequality for all control inputs $\bm{u} \in \mathcal{K}_{\mathrm{cbf},1}$,
\begin{equation*}
    \dot{\beta}_{1}(\bm{x}, \bm{u}) h_1(\bm{x}) - \beta_1(\bm{x}) \alpha_{1}(h_1(\bm{x})) - \dot{\beta}_{2}(\bm{x}, \bm{u}) > \alpha_2(h_2(\bm{x}))
\end{equation*}
there is no $\bm{u}$ satisfying $\dot{h}_{2}(\bm{x}, \bm{u}) \geq -\alpha_2(h_2(\bm{x}))$ in $\mathcal{K}_{\mathrm{cbf},1}(\bm{x})$. Therefore, any control inputs from $\mathcal{K}_{\mathrm{cbf},1}(\bm{x})$ do not belong to $\mathcal{K}_{\mathrm{cbf},2}(\bm{x})$.
\end{proof}
As noted in Proposition 1, the CBF-QP (Two-CBFs) problem is infeasible, which means that $\bm{u}^{\bullet}$ does not exist. To address this issue, we need to relax the QP problem, In the relaxation process, we introduce a new assumption:
\begin{assumption}
Under Assumption 1, the first safety condition is more critical than the other, so it must be strictly satisfied while tracking the trajectory, $\bm{x}(t) \in \mathcal{C}_{1}$ for all time. 
\end{assumption}
\noindent
Assuming that the conditions specified in Assumptions 1 and 2 are met, we present methods to obtain control inputs that prioritize the satisfaction of safety-critical constraints.

\section{Relaxation of CBF-QP}
\label{section3}
In this section, we present our proposed relaxation approaches for a scenario with two CBFs, Relaxed-CBF-QP and Hierarchical-CBF-QP. Then, we examine the impact of the weighting factors of QP formulations on the robot's behavior. In addition, we discuss a potential approach for determining the weighting factor of safety relaxation. We then generalize the approach to accommodate arbitrary CBFs. 

\begin{figure}[t] 
\centering
\includegraphics[width=\linewidth]{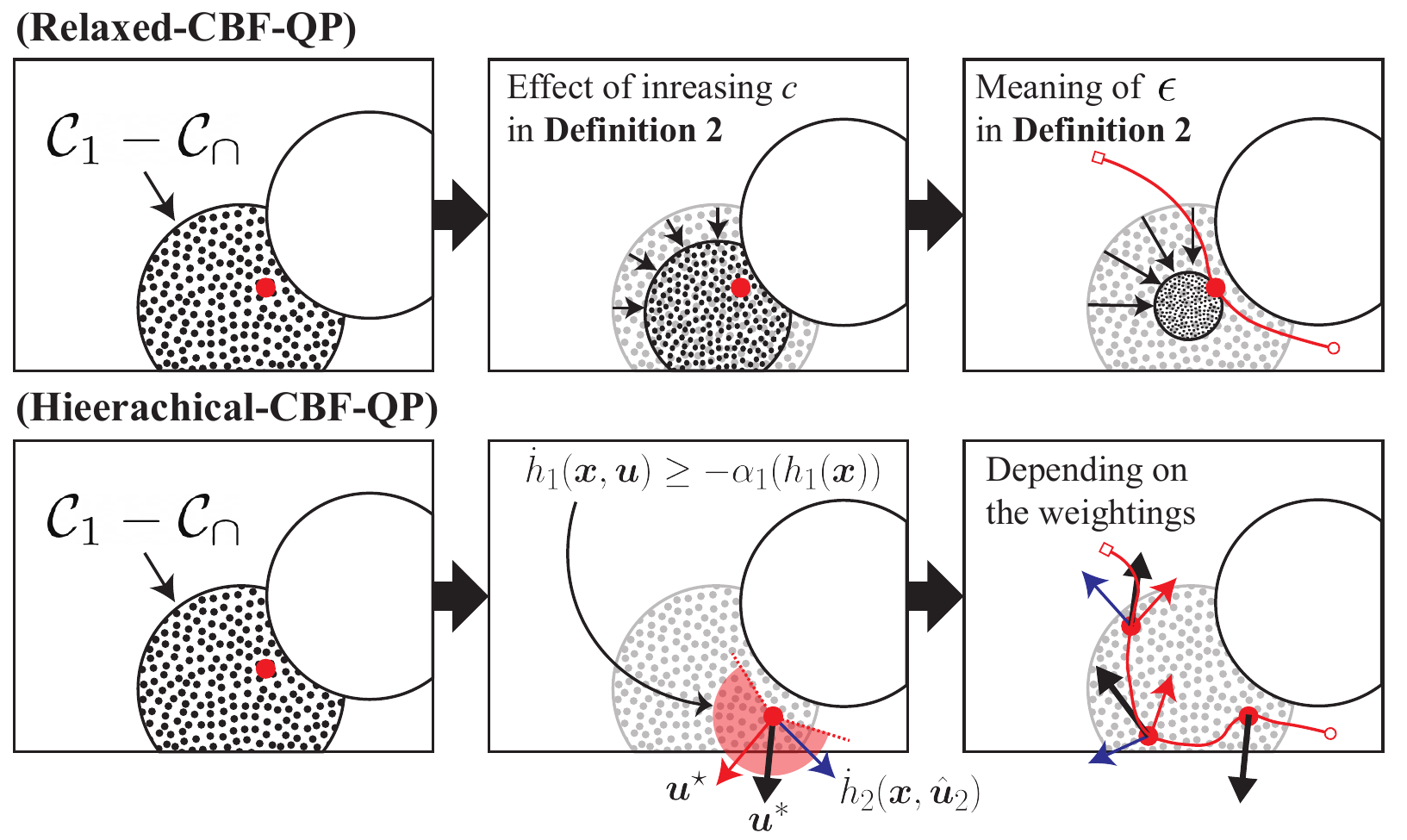}
\vspace{-6mm}
\caption{\textbf{Intuitions for the relaxation approaches}: Relaxed-CBF-QP (R-CBF-QP) aims to reduce the size of the safe set $\mathcal{C}_{1}- \mathcal{C}_{\cap}$ by adjusting the set $\mathcal{C}_2$. On the other hand, Hierarchical-CBF-QP (H-CBF-QP) combines the optimized control input $\bm{u}^{\star}$ with $\dot{h}_2(\bm{x}, \hat{\bm{u}}_2)$ to generate the appropriate control input.}
\label{Fig4}
\vspace{-0.5cm}
\end{figure}

\subsection{Two CBFs Case}
To ensure the feasibility of the CBF-QP problem, we modify the CBF constraint \eqref{cbf_2} under Assumptions 1 and 2. Specially, we define a minimum constant value $\epsilon$ for the secondary safety-critical constraint to ensure that  $\bm{x}\in \mathcal{C}_{\cap}$.
\begin{definition}
For a state $\bm{x}\in \mathcal{C}_1 - \mathcal{C}_{\cap}$, let $\mathbf{C}(\bm{x})$ be the set of constant values as follows:
\begin{equation}
    \mathbf{C}(\bm{x}) \coloneqq \{c \in \mathbb{R}_{\geq 0}: h_1(\bm{x}) \geq 0  \wedge h_2(\bm{x}) \geq -c \}.
\end{equation}
Then, the minimum value among the elements of $\mathbf{C}(\bm{x})$ is defined as $\epsilon \coloneqq \min \left( \mathbf{C}(\bm{x}) \right)$ as depicted in Fig. \ref{Fig4}.
\end{definition}
\noindent
In other words, if $h_2(\bm{x})$ is greater than or equal to $-\epsilon$, then the state $\bm{x}$ satisfies both safety conditions and is in the safe set $\mathcal{C}{\cap}$. Instead of using $h_2(\bm{x})$ in the CBF constraint, we introduce new barrier function with an offset as $h_2'(\bm{x}) = h_2(\bm{x}) + c$ where $\mathcal{C}_{2}' = \{\bm{x}\in \mathbb{R}^{n}: h_2'(\bm{x}) \geq 0 \}$. The secondary safety-critical constraints can then be expressed as following inequality constraints 
\begin{equation*}
\begin{split}
       \dot{h}_{2}'(\bm{x},\bm{u}) &= \mathcal{L}_f(h_2(\bm{x}) + c) +\mathcal{L}_g(h_2(\bm{x}) + c)\bm{u} \\
       &= \dot{h}_{2}(\bm{x}, \bm{u}) \geq - \alpha_{2}(h_2(\bm{x}) + c)
\end{split}       
\end{equation*}
where $\mathcal{L}_{f} c = 0$ and $\mathcal{L}_{g}c = 0$, respectively. We introduce an additional relaxation variable $\delta$ to make the CBF-QP problem feasible, which satisfies $\alpha_2( h_2(\bm{x}) + c) = \alpha_2(h_2(\bm{x})) + \delta$, which means 
\begin{equation*}
    c = \alpha_2^{-1}( \alpha_2(h_2(\bm{x})) + \delta ) - h_2(\bm{x}) \geq \epsilon.
\end{equation*}
If we already know $\epsilon$, then the relaxation variable is bounded as $\delta \geq \alpha_2(h_2(\bm{x}) + \epsilon) - \alpha_2(h_2(\bm{x}))$. Otherwise, we can simply set $\delta \geq 0$ in the optimization problem. Now, we formulate the first relaxed approach called Relaxed-CBF-QP as follows:
\begin{ruledtable}
\vspace{-2mm}
{\textbf{\normalsize Relaxed-CBF-QP (R-CBF-QP):}}
\par\vspace{-4mm}{\small 
\begin{subequations}
\begin{align}
    \bm{u} = \argmin_{\bm{u}\in \mathcal{U},\; \delta \in \mathbb{R}_{\geq0}} &\quad \mathcal{J}_{r} =  \lambda_{u}\| \bm{u}^{\star} - \bm{u} \|^2 + \lambda_{\delta}\delta^{2}, \nonumber \\
    \textrm{s.t.}&\quad \dot{h}_{1}(\bm{x}, \bm{u}) \geq - \alpha_{1}(h_1(\bm{x})), \label{cbf_1_new}\\
    &\quad \dot{h}_{2}(\bm{x}, \bm{u}) \geq - \alpha_{2} (h_{2}(\bm{x})) - \delta, \label{cbf_2_new}
\end{align}
\end{subequations}}
\vspace{-5mm}
\end{ruledtable}
\noindent
where $\lambda_{\delta} \geq 0$ and $\lambda_{u} \geq 0$ are the weighting factors. The solution to the above QP problem depends on the relative weighting factors in the cost function. 
\begin{remark} \label{remark01}
Let us consider $\lambda_{u}=0$ and $\lambda_{\delta}>0$. As per Definition 2 and the modified CBF constraint \eqref{cbf_2_new}, the solution to the R-CBF-QP problem satisfies \eqref{cbf_2_new} with the lower bound $\alpha_2(h_2(\bm{x}) + \epsilon)$. Therefore, we have $\delta^\star = \alpha_2(h_2(\bm{x}) + \epsilon) - \alpha_2(h_2(\bm{x}))$, which does not incorporate any feedback control effort.
\end{remark}

\begin{corollary}
When $\lambda_u>0$ and $\lambda_\delta=0$, the constraint \eqref{cbf_2_new} vanishes. Thus, the control input from the solution to R-CBF-QP is identical to one obtained by solving CBF-QP formulated in \eqref{opt_second} considering $h_1(\bm{x})$. 
\end{corollary}

As mentioned in the problem definition, our objective is to partially satisfy the secondary safety-critical constraints while strictly enforcing the primary constraint. So, we obtain $\hat{u}_{k}$ for each safety-critical constraint by solving \eqref{opt_second}. Then, we employ the solution $\hat{\bm{u}}_{k}$ to compute $\dot{h}_{k}(\bm{x}, \hat{\bm{u}}_{k})$. Next, we formulate a QP problem with an equality constraint that enforces the hierarchy between the safety-critical constraints under Assumption 2. 
\begin{ruledtable}
\vspace{-2mm}
{\textbf{\normalsize Hierarchical-CBF-QP (H-CBF-QP):}}
\par\vspace{-4mm}{\small 
\begin{subequations}
\begin{align}
    \bm{u}^{\ast} =\argmin_{\bm{u} \in \mathcal{U},\; \delta \in \mathbb{R}} &\quad  \mathcal{J}_{h} = \lambda_{u}\| \bm{u}^{\star} - \bm{u} \|^2 + \lambda_{\delta}\delta^{2}, \nonumber \\
    \textrm{s.t.}&\quad \dot{h}_1(\bm{x}, \bm{u}) \geq - \alpha_{1}(h_1(\bm{x})),\label{ieq_cbf_1} \\
    &\quad \dot{h}_{2}(\bm{x}, \bm{u}) + \delta =  \dot{h}_{2}(\bm{x},\hat{\bm{u}}_2), \label{eq_cbf_error}
\end{align}
\end{subequations}}
\vspace{-5mm}
\end{ruledtable}
\noindent
Because of Assumptions 1 and 2, $\delta$ cannot be zero in the H-CBF-QP. So, we enforce the constraint \eqref{eq_cbf_error} to minimize the difference between $\dot{h}_{2}(\bm{x}, \bm{u})$ and $\dot{h}_{2}(\bm{x}, \hat{\bm{u}}_2)$ as shown in Fig. \ref{Fig4}. 
\begin{corollary}
    Like Corollary 1, the constraint \eqref{eq_cbf_error} disappears when $\lambda_{u} >0$ and $\lambda_{\delta}=0$.   
\end{corollary}

\begin{proposition}
When $\lambda_u=0$ and $\lambda_{\delta} >0$, H-CBF-QP becomes \eqref{opt_second} with a cost $\mathcal{J}_{j} = \| \bm{u}^{\star} - \bm{u} \|_{\mathbf{M}}^2 + \bm{p}^{\top} \bm{u}$.
\end{proposition}
\begin{proof}
    By substituting $\delta$, the cost function of H-CBF-QP becomes $ \mathcal{J}_{h} = \| \dot{h}_{2}(\bm{x}, \hat{\bm{u}}_{2}) - \dot{h}_{2}(\bm{x}, \bm{u}) \|^{2}$. Let us consider $\hat{\bm{u}}_2 = \bm{u}^{\star} + \Delta \bm{u}$ from \eqref{opt_second}. Then, the cost function is specified as follows:
    \begin{equation*}
    \begin{split}
          \mathcal{J}_{h} =& (\bm{u}^{\star} + \Delta \bm{u} - \bm{u})^{\top}\mathbf{M}(\bm{u}^{\star} + \Delta \bm{u} - \bm{u}) \\
          =& \| \bm{u}^{\star} - \bm{u}\|_{\mathbf{M}}^{2}  +\bm{p}^{\top} \bm{u}  + \mathrm{const}    
    \end{split}
    \end{equation*}
    where $\bm{p}^{\top} = - 2\Delta \bm{u}^{\top} \mathbf{M}$ and $\mathbf{M} = \left( \mathcal{L}_{g}h_2(\bm{x})\right)^{\top}\mathcal{L}_{g}h_2(\bm{x})$, which is positive definite.
\end{proof}
\noindent
Different from R-CBF-QP, H-CBF-QP implicitly considers the feedback control input although $\lambda_{u} = 0$ due to Proposition 2. 

\subsection{Design of Weighing Factors}
The ideal approach would be to formulate and solve a large-scale nonlinear optimization problem that determines the weighting factors, but this is challenging since the entire control framework comprises multiple optimization problems in a cascaded structure. Therefore, in this section, we propose a straightforward method to find suitable parameters for the weightings. To simplify the problem, we fix the weighting factor for the feedback control input error as a constant $\lambda_{u} = 1$, and express the other weighting factor as a function of $h_2$. Specifically, we define $\lambda_{\delta}(h_2(\bm{x}))$ as:
\begin{equation}
    \lambda_{\delta}(h_2(\bm{x})) = \left\{ \begin{array}{ll} \gamma_{0} & \textrm{if} \quad h_2(\bm{x}) > 0,\\
    \gamma_{0}(1 + \Delta \gamma |h_2(\bm{x})|) & \textrm{else} \end{array} \right.
\end{equation}
where $\gamma_0>0$ is a baseline weighting factor for relaxing $h_2(\bm{x})$, and $\Delta \gamma>0$ is a gradient factor for preventing the violation of $h_2(\bm{x})\geq 0$. This weighting factor reduces $\delta$ more as $h_2(\bm{x})$ decreases below $0$. As discussed in Corollary 2 and Proposition 2, the proper selection of the parameters ($\gamma_0, \; \Delta \gamma$) is crucial to balance the performance of the feedback controller and the satisfaction of the lower-prioritized safety condition. 

One straightforward approach for searching the optimal parameter is to use a sampling-based technique. First, we define a cost function that depends on the task trajectory, which is discretized with a time interval $\Delta t$. The cost function is given as follows:
\begin{equation} \label{cost}
    \begin{split}
    r =&  \sum_{i=1}^{N-1} \| \bm{e}_x(i \Delta t)\|_{w_i} + \|\bm{e}_x(T) \|_{w_f} + w_{\delta} \sum_{i=1}^{N} \delta(i \Delta t)
    \end{split}
\end{equation}
where $w_i$ and $w_f$ are the weightings for the tracking trajectory and achieving the final goal, respectively, and $w_f \gg w_i >0$. $w_{\delta} >0$ is a discount factor that penalizes safety relaxation of the secondary condition. Second, we randomly draw samples from the Normal distribution $(\gamma_0,\; \Delta \gamma) \sim \mathcal{N}(\bm{\mu},; \bm{\Sigma})$, where $\bm{\mu}$ and $\bm{\Sigma}$ are the mean vector and covariance matrix, respectively. We iterate the simulations with the random samples to find the optimal pair $(\gamma_0,\;\Delta \gamma)$ that minimizes the cost function \eqref{cost}. When designing the weightings by evaluating the cost, we eliminate the offset as $ r_{o}(\Delta \gamma,\gamma_0) = r(\Delta \gamma, \gamma_0) - \min(r)$ where $r(\Delta \gamma, \gamma_0)$ denotes the cost with $\Delta \gamma$ and $\gamma_0$.

\subsection{Generalization}
We generalize the proposed approach for handling general cases to incorporate $m$ CBF constraints. To achieve this, we rewrite Assumptions 1 and 2 as follows for the extensive case:
\begin{assumption}
    Given $m$ CBFs, $h_1,\; \cdots, \;h_m$, we assume that a state $\bm{x}$ belongs to $\mathcal{C}_1 - \mathcal{C}_{\cap}$ where $\mathcal{C}_{\cap} = \{\bm{x} \in \mathbb{R}^{n}: h_1(\bm{x}) \geq 0 \wedge \cdots \wedge h_{m}(\bm{x}) \geq 0\}$. The importance of the CBFs is determined in lexicographic order.  
\end{assumption}
Under the above assumption, it is possible to extend our approach in both implicity and explicit ways. First, we formulate a single QP problem that considers the feedback control input and all safety-critical constraints implicitly. This is the Implicit-H-CBF-QP (IH-CBF-QP) and is given as follows:
\begin{ruledtable}
\vspace{-2mm}
{\textbf{\normalsize Implicit-H-CBF-QP (IH-CBF-QP):}}
\par\vspace{-4mm}{\small 
\begin{subequations}
\begin{align*}
    \min_{\bm{u}\in \mathcal{U},\; \delta_{2}, \cdots, \delta_m \in \mathbb{R}} &\quad \mathcal{J}_{imp} = \lambda_{u}\| \bm{u}^{\star} - \bm{u} \|^2 + \sum_{k=2}^{m }\lambda_{\delta_{k}}\delta_{k}^{2}, \nonumber \\
    \textrm{s.t.}&\quad \dot{h}_1(\bm{x}, \bm{u}) \geq - \alpha_1(h_1(\bm{x})),\\
    &\quad \dot{h}_{k}(\bm{x}, \bm{u}) +  \delta_k = \dot{h}_{i}(\bm{x}, \hat{\bm{u}}_{k}), \quad \forall i \in \{2, \cdots, m\}
\end{align*}
\end{subequations}}
\vspace{-5mm}
\end{ruledtable}
\noindent
We obtain a feasible control input by solving $m$ QP problems, which are \eqref{opt_second} for $k=2$ to $k=m$ and the above IH-CBF-QP. However, the solution to the problem is too sensitive and ambiguous to the weighting parameters in the cost function. Since all relaxation variables are coupled, it is hard to impose the hierarchy among them in this implicit way, even if heuristics are used. Improper weightings may result in reducing the lower-prioritized constraint and increasing the error of the higher-prioritized one. 

 \begin{figure*}[t] 
\centering
\includegraphics[width=.95\linewidth]{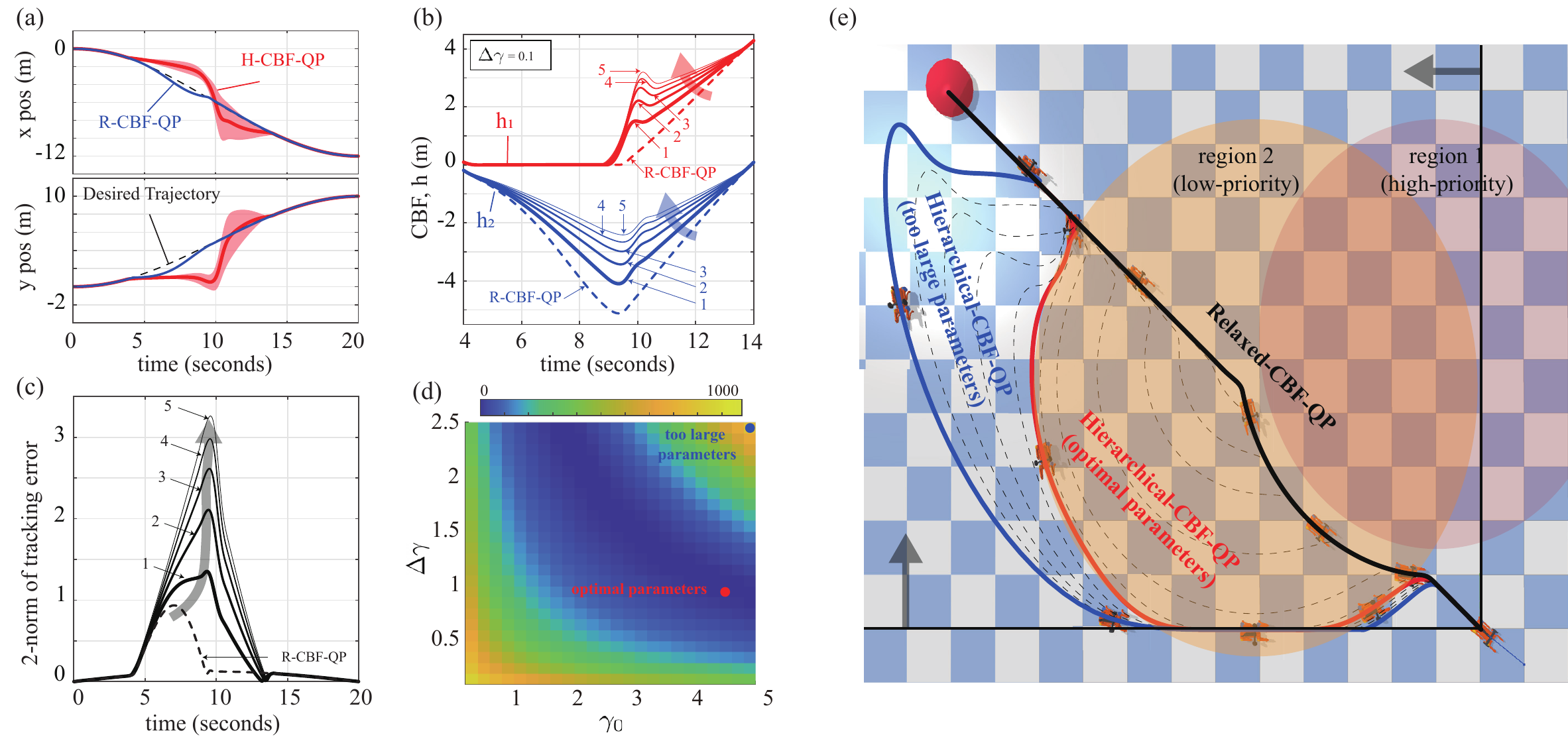}
\vspace{-4mm}
\caption{\textbf{Simulation results}: (a) $x$ and $y$ positions of the robot, (b) CBF values with respect to different values of $\gamma_0$, (c) tracking errors with respect to different values of $\gamma_0$, (d) color map of the cost values in terms of the sampled parameters, (e) robotic simulation results demonstrated using a quadruped robot (A1) in 2D operational space. In figure (a), we present the entire results of H-CBF-QP given the sampled $\gamma_0$ and $\Delta \gamma$. Figures (b) and (c) include the results with the fixed $\Delta \gamma$ for showing a clear tendency of $\gamma_0\; (1,\; 2,\; 3,\; 4,\; 5)$.}
\label{Fig5}
\vspace{-0.55cm}
\end{figure*}

To address this issue, we introduce a sequential optimization method, Explicit-H-CBF-QP (EH-CBF-QP). This formulation is based on a recursive extension of the H-CBF-QP and defined as follows for $k\geq 3$:
\begin{ruledtable}
\vspace{-2mm}
{\textbf{\normalsize Explicit-H-CBF-QP (EH-CBF-QP):}}
\par\vspace{-4mm}{\small 
\begin{subequations}
\begin{align*}
    \min_{\bm{u}\in \mathcal{U},\; \delta_{k}\in \mathbb{R}} &\quad \mathcal{J}_{exp} = \lambda_{u}\| \bm{u}^{\star} - \bm{u} \|^2 + \lambda_{\delta_{k}}\delta_{k}^{2}, \nonumber \\
    \textrm{s.t.}&\quad \dot{h}_1(\bm{x}, \bm{u}) \geq - \alpha_1(h_1(\bm{x})),\\
    &\quad \dot{h}_{i}(\bm{x}, \bm{u}) +  \delta_i^{\ast} = \dot{h}_{i}(\bm{x}, \hat{\bm{u}}_{i}), \quad \forall i \in \{2, \cdots, k-1\}\\
    &\quad \dot{h}_{k}(\bm{x}, \bm{u}) +\delta_k= \dot{h}_{k}(\bm{x}, \hat{\bm{u}}_{k}) , 
\end{align*}
\end{subequations}}
\vspace{-5mm}
\end{ruledtable}
\noindent
Here, $\delta_{i}^{\ast}$ denotes the optimal relaxation variable obtained from the optimization of the $i$-th H-CBF-QP. By recursively solving the EH-CBF-QP problems from $k=2$ to $k=m$, we obtain the optimal control input and relaxation variables for all safety-critical constraints. 
\begin{corollary}
    In the EH-CBF-QP formulation, the hierarchy is strictly maintained, i.e., it is impossible to decrease $\delta_{k}^{2}$ without increasing $\delta_{i}^{2}$ with higher priority ($i<k$). 
\end{corollary}

\section{Implementation}
\label{section4}
In this section, we present simulation and experiment results that demonstrate the effectiveness of our approach using a quadruped robot system (A1). 

\subsection{Simulation Scenario}
The simulations were performed on a laptop equipped with a 3.4 GHz Intel Core i7 processor using MATLAB, Pybullet, and OSQP software tools. In addition, we consider a double integrator (ROM) and PD controller to validate our approach in the high-level architecture, then implement a full-body task space controller to realize the control input of ROM in the simulation. The desired orientation of the robot's base is computed based on the ratio of velocity components in the $x$ and $y$ directions. The state of ROM is defined as $\bm{x}=[\bm{p}^{\top},\; \dot{\bm{p}}^{\top}]^{\top}$ where $ \bm{p} \in \mathbb{R}^{2}$ represents the position in a $2$-dimensional space, and the control input $\bm{u}\in \mathcal{U}$ is bounded by an input set $\mathcal{U}=\{ \bm{z} \in \mathbb{R}^{2}: |z_i|\leq 5, \; i =1,2  \}$.

In this paper, we aim to solve a tracking problem where the task trajectory, $\bm{p}^{d}(t)$, is given for $[0,\; 20]$ seconds with $\bm{p}^{d}(0) = [0,\;0]^{\top}$ and $\bm{p}^{d}(20) =[-12,\; 10]^{\top}$. To ensure safety, we enforce inequality constraints such that $h_1(\bm{x}) = p_x + 0.2 \geq 0$, $h_2(\bm{x}) = p_y +0.2 \geq 0$, and
\begin{equation*}
    h_{k} = \| \bm{p} - \bm{p}_{k} \| - r_{k} \geq 0
\end{equation*}
where $\bm{p}_{3} = [-1,\; 5.5]^{\top}$, $\bm{p}_{4} = [-1,\; 4.5]^{\top}$, $r_3 = 4$, and $r_4 = 5$. In this simulation, we assume that the region represented by $h_3(\bm{x})\geq 0$ (region 1) is the more dangerous region than the region represented by $h_4(\bm{x}) \geq 0 $ (region 2), and hence, a hierarchy among these safety constraints is imposed in lexicographic order. For the class $\mathcal{K}$ functions, we utilize the following function:
\begin{equation*}
    \alpha(h) = \left\{\begin{array}{cc} \varphi_{1} \log (\varphi_{2} h + \varphi_3) - \varphi_1 \log \varphi_3  & h\geq0, \\
    - \varphi_{1}\log(-\varphi_2h + \varphi_3) + \varphi_1 \log \varphi_3 & h<0 \end{array} \right.
\end{equation*}
where $\varphi_1, \varphi_2, \varphi_3 >0$. In this section, we set $\varphi_1 = 1$, $\varphi_2 = 0.5$, and $\varphi_3 = 0.1$. We draw random samples for $\delta_2$ with $\bm{\mu}=[3,\;1.5]$ and $\bm{\Sigma} = \textrm{diag}(1,\;1)$. Next, we consider simulation results with the samples in $\gamma_0 \in [0,\;5]$ and $\Delta \gamma \in [0,\;2.5]$ for simplicity. For the cost \eqref{cost}, the weightings are set as $w_i = 1$, $w_f=1000$, $w_\delta = 0.5$, respectively.



\begin{figure*}[t] 
\centering
\includegraphics[width=\linewidth]{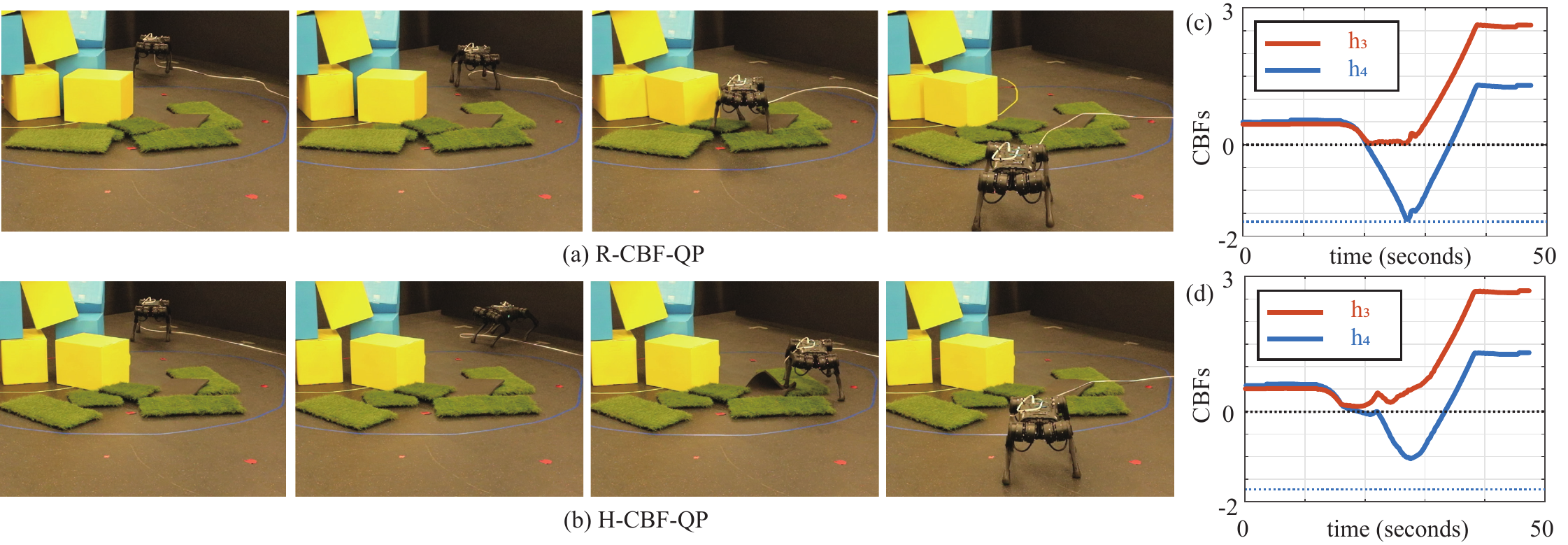}
\vspace{-7mm}
\caption{\textbf{Snapshots of experimental validation}: (a) and (b) show the snapshots of experiments implemented by using R-CBF-QP and H-CBF-QP, respectively. (c) and (d) represent CBF value variations in R-CBF-QP and H-CBF-QP, respectively. The robot successfully reaches the goal destination using both approaches. However, there is a collision between the robot and one of the yellow boxes, which is located near the center of Region 2, in the third snapshot of (a). In contrast, H-CBF-QP effectively avoids the collision. }
\label{Fig6}
\vspace{-0.5cm}
\end{figure*}

\subsection{Simulation Results}
We simulated the R-CBF-QP and H-CBF-QP controllers with different parameters for the weighting factor, and the results are depicted in Fig. \ref{Fig5}. Fig. \ref{Fig5}(a) shows the desired and actual base positions of the robot with different controllers. While R-CBF-QP prioritizes safety-critical constraints for region 1, the actual path deviates slightly from the desired trajectory. In contrast, the robot controlled by H-CBF-QP deviates more far from the desired trajectory than the one controlled by R-CBF-QP because the robot tries to become safer by escaping region 2. Fig \ref{Fig5}(b) shows that the barrier function value $h_4(\bm{x})$ becomes closer to zero while maintaining $h_3(\bm{x})$ above zero, indicating that our approach is capable of reducing the violations of secondary safety conditions while strictly satisfying those with the highest priority. However, as shown in Fig. \ref{Fig5}(c), the tracking error of the feedback controller increases as we attempt to make the robot safer in terms of the secondary safety condition. To balance the feedback control performance and safety, we deploy the proposed a sampling-based technique to find proper parameters and analyze them in a cost map of Fig. \ref{Fig5}(d).

Fig. \ref{Fig5}(e) shows the simulation results implemented by using Pybullet. The black path indicating the result driven by R-CBF-QP goes through the center of region 2. However, the robot controlled by H-CBF-QP tries to avoid region 2, even if it incurs some tracking errors. Interestingly, if we set the weighting parameters too large, the robot moves far from the desired trajectory, and it becomes challenging to return to the original path. Therefore, we need to find the proper balance between performance and safety parameters.

\subsection{Experimental Validation}
We demonstrate the effectiveness and efficiency of our proposed approaches using the real quadruped robot (A1). We maintain consistency between the simulation scenario and the experimental setup. To do experiments in a realistic environment, we construct an experimental structure using boxes (region 1) and artificial turf plates (region 2) shown in Fig. \ref{Fig6}. At the center of the region 2, we put a box which is lighter than ones in the region 1. Given the assumption that the boxes in the region 1 are too heavy for the robot to move, we prioritize avoiding region 1 over region 2. Due to space limitations, we scale down the problem setup and set the centers of the regions at $\bm{p}_{3} = [ -1.7,\; 1.5 ]^{\top}$ and $\bm{p}_{4} = [-0.35,\; 1.8]^{\top}$, with radii of $r_3=1.7$ and $r_4 = 1.34$, respectively. The goal destination is set at $[-4,\; 3.34]^{\top}$, and we generate a continuous trajectory toward the destination during a $20$-second interval. The detailed specifications for the CBFs are the same as ones in the simulations. By demonstrating simulations with the sampled parameters, then, we obtain the pair for H-CBF-QP as $(5,\;1.3)$. 

In these experiments, both relaxation approaches, R-CBF-QP and H-CBF-QP, are able to generate feasible control inputs for the robot. Despite colliding with the box located at the center of region 2, R-CBF-QP is able to steer the robot towards the goal. However, if the box becomes much heavier than our setup, the robot may not be able to reach the goal destination. In contrast, H-CBF-QP generates a safer behavior and enables the robot to avoid the yellow box at the center of region 2, as shown in Fig. \ref{Fig7}. Fig. \ref{Fig6}(c) and (d) confirm that H-CBF-QP results in less violations of the safety condition associated with $h_4$ compared to R-CBF-QP.  

\begin{figure}[t] 
\centering
\includegraphics[width=\linewidth]{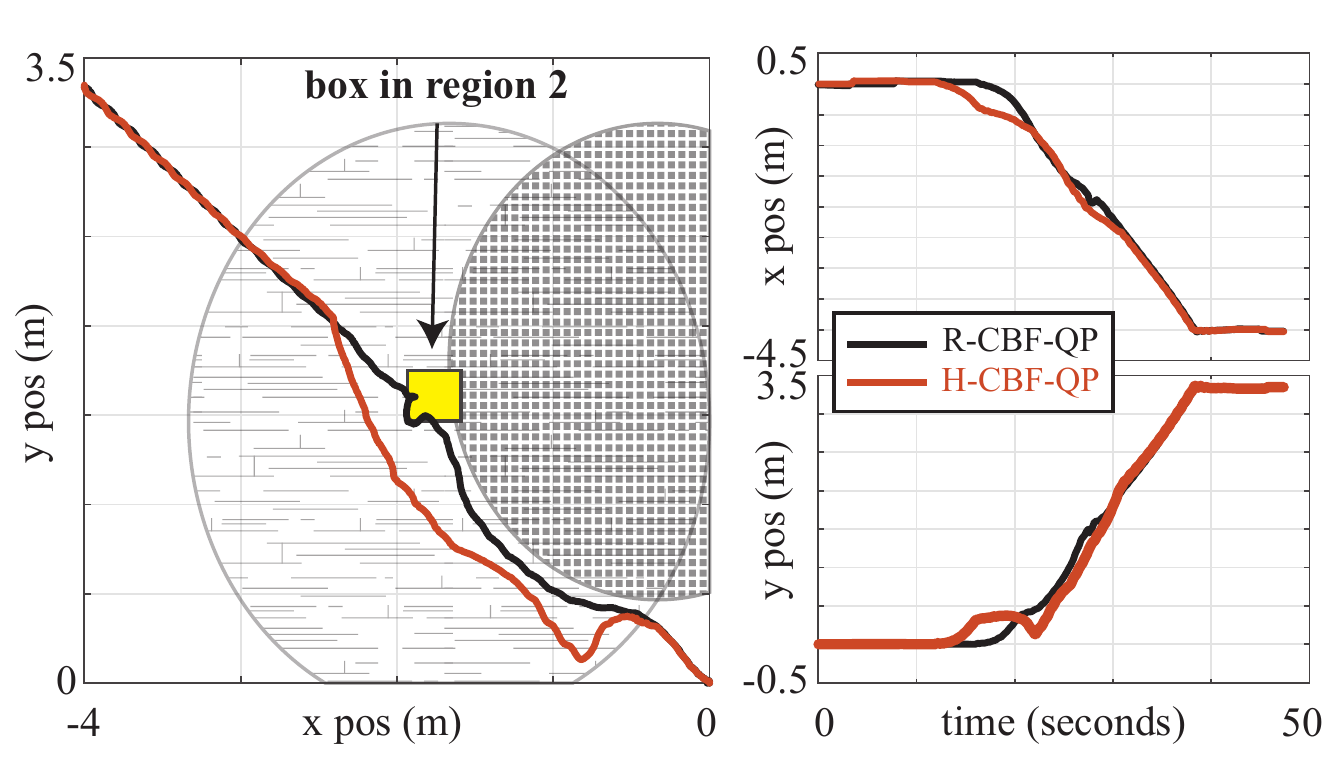}
\vspace{-7mm}
\caption{\textbf{Experimental results}: The movements of the robot in the operational space (left) and the position over the defined time horizon (right). There exists a box at the center of Region 2 depicted by a yellow square. }
\label{Fig7}
\vspace{-5mm}
\end{figure}

\section{Conclusion}
This paper presents a novel and intuitive approach to address contradictory safety conditions in robotic systems. By imposing a hierarchy among these conditions, we formulate feasible QP problems that consider the hierarchical safety conditions while generating dynamic locomotion. R-CBF-QP generates a feasible control input, taking into account only the highest priority safety constraint. On the other hand, the proposed H-CBF-QP is capable of strictly holding the safety condition with the highest priority, while handling others with lower priority. We also leverage a sampling-based method to determine the weighting factors in the H-CBF-QP formulation. The effectiveness of our approach is demonstrated through simulations in achieving safer locomotion of the quadruped robot (A1).

In the future, we aim to expand our proposed approach to handle more intricate physical interactions between robots and their environments. Our focus is on incorporating physical reaction forces when walking on contact-rich terrain or interacting with other agents in obstacle-rich environments. Furthermore, we are exploring the theoretical aspect of Input-to-State safety verification of our method and assessing the impact of uncertainties arising from estimations.

\bibliographystyle{IEEEtran}
\bibliography{l_css}

\end{document}